\documentclass[11pt]{article}

\usepackage[letterpaper,left=1in,right=1in,top=1in,bottom=1in]{geometry}
\usepackage[T1]{fontenc}
\usepackage{mathptmx}
\usepackage{courier}
\usepackage{amsmath, amssymb, amsthm}
\usepackage{booktabs}
\usepackage{colortbl}
\usepackage{enumitem}
\usepackage{graphicx}
\usepackage{float}
\usepackage[font=small,labelfont=bf]{caption}
\usepackage{microtype}
\usepackage{titlesec}
\usepackage{tikz}
\usepackage{varwidth}
\usetikzlibrary{arrows.meta, positioning, fit, backgrounds}
\usepackage[numbers,sort&compress]{natbib}
\usepackage[colorlinks=true,linkcolor=black,citecolor=black,urlcolor=black]{hyperref}
\usepackage[most]{tcolorbox}

\setlist[enumerate]{topsep=0.25em,itemsep=0.2em,leftmargin=1.6em}
\titlespacing*{\section}{0pt}{1.1em}{0.45em}
\titlespacing*{\subsection}{0pt}{0.8em}{0.3em}
\titleformat{\section}{\normalfont\large\bfseries}{\thesection}{0.8em}{}
\titleformat{\subsection}{\normalfont\normalsize\bfseries}{\thesubsection}{0.8em}{}

\newtheorem{proposition}{Proposition}

\newtheorem{corollary}{Corollary}

\newcommand{\E}{\mathbb{E}}
\newcommand{\Prob}{\mathbb{P}}
\newcommand{\Var}{\operatorname{Var}}
\newcommand{\Cov}{\operatorname{Cov}}
\newcommand{\Corr}{\operatorname{Corr}}
\newcommand{\neff}{n_{\mathrm{eff}}}
\newcommand{\deff}{d_{\mathrm{eff}}}

\definecolor{cnavy}{HTML}{1F4E79}
\definecolor{cink}{HTML}{16202B}
\definecolor{cpanel}{HTML}{EEF1F4}
\definecolor{carrow}{HTML}{8A94A0}

\newtcolorbox{takeaway}{enhanced, center, width=0.92\linewidth, colback=cpanel, colframe=cnavy,
  boxrule=0pt, leftrule=2.4pt, arc=2pt, left=11pt, right=11pt, top=7pt, bottom=7pt,
  fontupper=\small, before skip=0.8\baselineskip, after skip=0.8\baselineskip}

\newcommand{\papertitle}{When More Sampling Hurts: \\The Modal Ceiling and Correlation Ceiling of Test-Time Scaling}
\hypersetup{pdftitle={\papertitle}, pdfauthor={Yong Yi Bay and Kathleen A. Yearick}, pdfkeywords={test-time scaling, inference-time compute, repeated sampling, pass@k, coverage, best-of-n, self-consistency, effective sample size, design effect, intraclass correlation, correlation ceiling, modal ceiling, identifiability gap, reasoning models}}
\newsavebox{\papertitlebox}
\newcommand{\paperfrontmatter}{%
\begin{center}
\vspace*{-0.62in}
\begin{lrbox}{\papertitlebox}%
\begin{varwidth}{\textwidth}
\centering
{\Large\bfseries \papertitle\par}
\end{varwidth}%
\end{lrbox}%
\rule{\wd\papertitlebox}{1.1pt}\par
\vspace{0.55em}
\usebox{\papertitlebox}\par
\vspace{0.55em}
\rule{\wd\papertitlebox}{1.1pt}\par
\vspace{1.7em}
{\textbf{Yong Yi Bay}\textsuperscript{*}\hspace{2.4em}\textbf{Kathleen A. Yearick}\textsuperscript{*}\par}
\vspace{0.5em}
PhD, University of Illinois at Urbana-Champaign\par
\let\thefootnote\relax\footnotetext{\textsuperscript{*}Equal contribution. Correspondence: \texttt{\{yongyibay, kallie.a.yearick\}@gmail.com}.}
\vspace{1.9em}
\end{center}
}

\newenvironment{paperabstract}
{\begin{center}\textbf{\textsc{\textls[120]{Abstract}}}\end{center}
\begin{list}{}{\leftmargin=0.78in\rightmargin=0.78in}\item\small}
{\end{list}\normalsize}

\newcommand{\paperkeywords}[1]{\noindent\textbf{\textit{Keywords}} #1\par\vspace{0.65em}}

\begin{document}
\paperfrontmatter

\begin{paperabstract}
People overthink; language models over-sample, and the extra effort can talk both into a \emph{worse} answer. Reasoning systems answer a hard question by sampling it many times (\emph{test-time scaling}), and the more they draw, the more often a correct answer turns up somewhere, so \emph{coverage}, the fraction of problems with at least one correct try, climbs and appears to be progress. But a deployed system must return one answer, and choosing it, not knowing which try is right, is \emph{selection}; selection is capped, and past a point extra samples only make the model surer of a confident mistake, even as every draw adds cost. The gap between climbing coverage and stalled selection, the \emph{identifiability gap}, is the answer a model can produce but not pick. So the real question is not whether to sample but how far, and the answer is: not far. For picking an answer, the vote has already settled within a few dozen draws, the \emph{modal ceiling}; for scoring a benchmark, sooner still, the \emph{correlation ceiling}. Beyond that, extra draws cost compute and add nothing, and can even make the answer worse. This paper turns the cutoff into a single number, the \emph{effective number of samples}, that any sampling run already reveals. The bottleneck is recognizing a right answer, not generating one.
\end{paperabstract}

\paperkeywords{test-time scaling $\cdot$ inference-time compute $\cdot$ repeated sampling $\cdot$ pass@k $\cdot$ coverage $\cdot$ best-of-n $\cdot$ self-consistency $\cdot$ effective sample size $\cdot$ design effect $\cdot$ intraclass correlation $\cdot$ correlation ceiling $\cdot$ modal ceiling $\cdot$ identifiability gap}

\section{Introduction and roadmap}

A reasoning model can be made stronger at inference time by spending more compute, and the compute goes to one of two levers: longer reasoning, letting the model think more before it answers, or more sampling, drawing $n$ answers to one prompt and combining them \citep{wang2023selfconsistency,brown2024monkeys,snell2024scaling,wu2024inference}. This is test-time scaling, and it drives much of the progress in reasoning systems \citep{openai2024o1,deepseekai2025r1,muennighoff2025s1,zhang2025ttssurvey}; the sampling lever is the subject here. Its headline curve is \emph{coverage}, the fraction of problems on which at least one of $n$ samples is correct \citep{chen2021codex}, and it climbs over several orders of magnitude of $n$ \citep{brown2024monkeys}, reading as steadily rising capability. That reading is too generous. A deployed system must return a single answer, and with nothing to certify which sample is right it can only \emph{select} one, by frequency or a learned score; selection is capped, and more sampling does not lift it and can even lower it. A correct answer thus grows ever easier to reach but no easier to return: the bottleneck in test-time scaling is not generating a right answer, it is recognizing one. And because each draw costs compute, the question that matters is how far to sample, which the ceilings below answer with a small, goal-dependent budget (Figure~\ref{fig:two_ceilings}).

\begin{figure}[H]
\centering
\includegraphics[width=\linewidth]{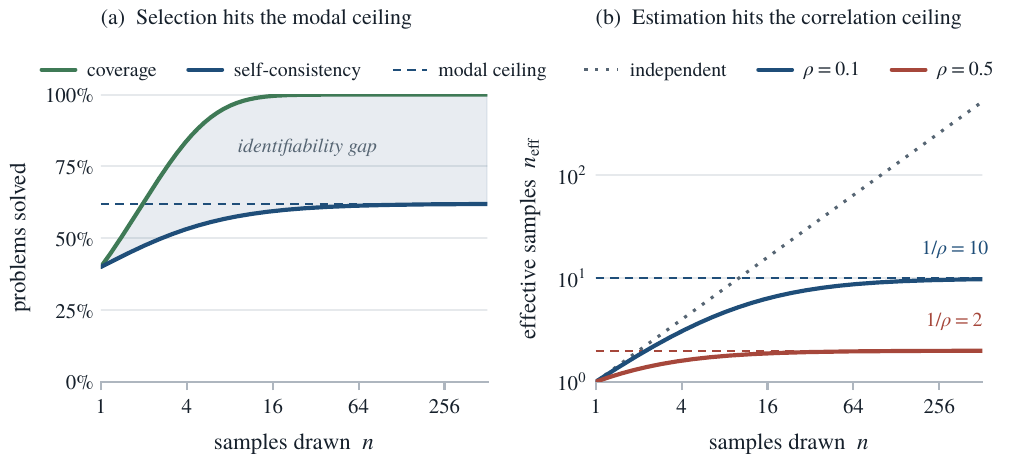}
\caption{The two ceilings of test-time sampling. \textbf{(a)} Coverage, self-consistency, and the identifiability gap between them. \textbf{(b)} The effective number of samples $n_{\mathrm{eff}}$ against the nominal count $n$, with the correlation ceiling $1/\rho$.}
\label{fig:two_ceilings}
\end{figure}

\textbf{The split.} Coverage and selection ask different questions of the same samples, and only one of them keeps improving. Coverage asks whether \emph{any} sample is correct; a verifier that could spot the correct one would let coverage ride past every limit toward what the model can ever reach. Selection must commit to one answer with no such verifier, and a plurality vote converges to the model's most common answer; on the problems where that answer is wrong, more samples only make the wrong answer win more surely, so selection saturates at the fraction of problems whose most common answer is correct, and can even decline as the budget grows. \citet{brown2024monkeys} report exactly this pattern without a mechanism, that ``majority voting and reward models plateau beyond several hundred samples and fail to fully scale with the sample budget.'' The distance between rising coverage and stalled selection is the set of problems the model can solve but cannot pick out, the \emph{identifiability gap}: it can generate the right answer without being able to select it (Figure~\ref{fig:concept}).

\begin{figure}[H]
\centering
\includegraphics[width=0.86\linewidth]{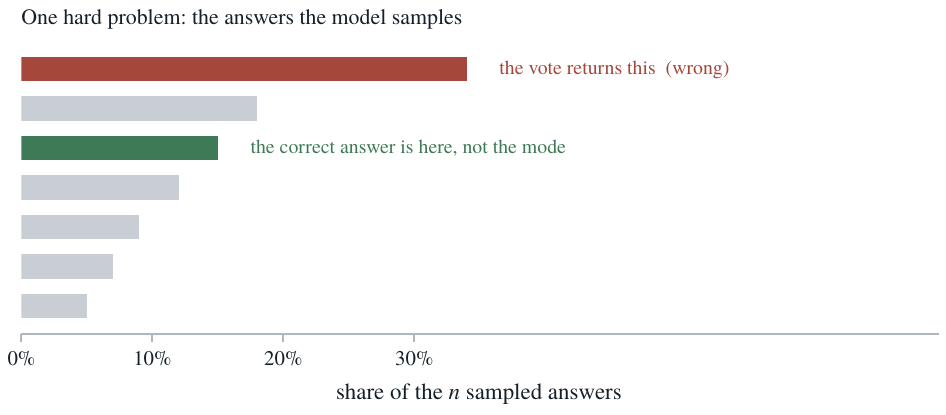}
\caption{Generated but not selected. On a hard problem the correct answer appears among the sampled attempts, so coverage finds it; but it is not the most common answer, so the plurality vote returns a confident wrong one.}
\label{fig:concept}
\end{figure}

\textbf{A second, separate ceiling.} Even the benchmark accuracy that sampling is meant to estimate is worth less than the sample count suggests, for a different reason. The $n$ attempts to a problem are not $n$ independent tries; they are a cluster drawn from one prompt, alike the way several people from one household are alike in a survey. The classical correction is the \emph{design effect} $\deff = 1+(n-1)\rho$ of \citet{kish1965survey}, where $\rho$ is the intraclass correlation among the attempts, so the $n$ correlated draws are worth only $n/\deff$ independent ones \citep{cochran1977sampling}. Read for test-time sampling, this is the \emph{effective number of samples}
\begin{equation}
\neff = \frac{n}{1 + (n-1)\rho},
\qquad
\neff \;\longrightarrow\; \frac{1}{\rho}\quad\text{as } n\to\infty ,
\label{eq:neff_intro}
\end{equation}
which saturates at a hard \emph{correlation ceiling} $1/\rho$: no budget makes $n$ correlated attempts worth more than $1/\rho$ independent ones (Section~\ref{sec:scope}). This ceiling governs estimation, not selection; the two are kept apart throughout. The identifiability gap is measured in Section~\ref{sec:gap}.

\textbf{What is and is not new.} The design effect and effective sample size are not new here; they are \citet{kish1965survey} and \citet{cochran1977sampling}, and the correction of Condorcet's theorem for correlated voters is older still \citep{ladha1992condorcet,boland1989majority}. Three recent works apply the same instrument to language-model outputs, but to different objects: \citet{kohli2026judges} measure the effective votes of a panel of \emph{distinct judge models} in evaluation, \citet{goel2025alike} quantify error correlation \emph{across models}, and \citet{nitarach2026aimo}, in a competition report, track the effective sample size of majority voting across mixed models. None treats single-model test-time scaling as cluster sampling, separates the estimation ceiling from a distinct selection ceiling, derives the modal-answer ceiling and its anti-scaling, or connects to the survey-sampling literature. On the coverage side, \citet{schaeffer2025powerlaws} and \citet{kazdan2025predicting} explain the power-law shape of coverage by a heavy-tailed distribution of \emph{per-problem} difficulty, and \citet{levi2024simplemodel} reaches a power law through a memorization ansatz; all assume attempts are conditionally independent given the problem, the within-problem $\rho_w=0$ case of the model below, from which the analysis here recovers a power law of the same form. The contribution here is the lens that separates the three quantities a nominal sample count confounds: a between-problem difficulty spread $\rho_b$ that caps benchmark estimation and shapes coverage, a within-problem dependence measured to be near zero, and the concentration of the answer distribution that caps selection at a modal-hit rate $\pi_{\mathrm{mode}}$ and makes it anti-scale, all grounded on released logs.

This work is a citable derivation, not a new algorithm. Section~\ref{sec:cluster} sets up test-time sampling as cluster sampling and fences the scope of the exact claims. Section~\ref{sec:neff} derives the effective number of samples, the correlation ceiling, and the marginal value of a sample. Section~\ref{sec:cov_sel} separates coverage from selection, explains the identifiability gap between them, and measures it on both an independent-draw log \citep{brown2024monkeys} and a dependent-draw log \citep{beeching2024scaling}. Section~\ref{sec:practice} decomposes within- and between-problem correlation, gives the compute-allocation rule, an estimator for $\rho$, and a summary table.

\section{Test-time sampling is cluster sampling}
\label{sec:cluster}

The entire argument rests on one reframing: a problem and its repeated attempts form a cluster, not a fresh draw each time. This section makes the correspondence precise and fences exactly where the claims that follow are exact.

Fix a prompt $q$ and draw $n$ responses $o_1,\dots,o_n$ from one model at a fixed decoding configuration (the same sampling settings each time). A verifier (an automatic checker that marks each answer right or wrong) scores each, giving binary success indicators
\begin{equation}
Y_i = \mathbf{1}\{\,o_i \text{ is correct for } q\,\} \in \{0,1\},
\qquad i = 1,\dots,n .
\label{eq:Y}
\end{equation}
Write $s = \Prob[Y_i = 1]$ for the per-attempt success probability and $K = \sum_{i=1}^n Y_i$ for the number correct. The three headline quantities of test-time scaling are functions of $(Y_1,\dots,Y_n)$:
\begin{equation}
\underbrace{\mathrm{pass}@n = \Prob[K \ge 1]\rule[-5pt]{0pt}{0pt}}_{\text{coverage}},
\qquad
\underbrace{\hat p = K/n\rule[-5pt]{0pt}{0pt}}_{\text{success fraction}},
\qquad
\underbrace{\mathbf{1}\{K > n/2\}\rule[-5pt]{0pt}{0pt}}_{\text{majority vote}} .
\label{eq:estimands}
\end{equation}
Best-of-$n$ (drawing $n$ answers and returning the one a learned reward model scores highest) \citep{cobbe2021gsm8k}, weighted voting, and self-consistency (returning the answer the samples agree on most often) are all read off the sampled distribution and so are functions of the same draws.

The independence baseline assumes $Y_1,\dots,Y_n$ are i.i.d.\ (independent and identically distributed) Bernoulli$(s)$ draws, each a coin that lands correct with probability $s$. Then $\mathrm{pass}@n = 1-(1-s)^n$, $\Var(\hat p) = s(1-s)/n$, and, for $s>\tfrac12$, the majority vote is correct with probability tending to $1$ as $n\to\infty$ \citep{chen2021codex}. The analysis keeps the marginal $s$ and replaces independence with exchangeability.

\textbf{Exchangeable attempts.} The attempts to one problem are exchangeable: relabeling them does not change their joint distribution, so only how many succeed matters, not which ones. By de Finetti's theorem \citep{definetti1937prevision}, an infinitely exchangeable binary sequence (one a fresh session can in principle extend to arbitrarily many attempts) is a mixture of i.i.d.\ sequences,
\begin{equation}
Y_i \mid \theta \;\overset{\text{i.i.d.}}{\sim}\; \text{Bernoulli}(\theta),
\qquad
\theta \sim G \text{ on } [0,1],
\label{eq:definetti}
\end{equation}
for some mixing distribution $G$ (the spread of these hidden success rates across problems) with mean $\E[\theta] = s$. In words, the model behaves as if each fresh session first draws a hidden success rate $\theta$ for the problem, then every attempt in that session is an independent $\theta$-weighted coin: a good session lands in a strong reasoning basin (large $\theta$), a weak one in a poor basin (small $\theta$). The single number that summarizes the dependence is the \emph{intraclass correlation} $\rho$, how strongly two attempts on the same problem move together, and the standard measure for clustered binary data,
\begin{equation}
\rho \;=\; \Corr(Y_i, Y_j) \;=\; \frac{\Var(\theta)}{s(1-s)} \;\in\; [0,1],
\qquad i \ne j .
\label{eq:icc}
\end{equation}
The representation forces $\Var(\theta)\ge0$, so $\rho\ge0$: exchangeable attempts are non-negatively correlated. Independence is $\rho = 0$ (a degenerate $G$ at $s$); $\rho = 1$ is total collapse, where a session is all correct or all wrong together. Equation~\eqref{eq:icc} is the bridge from the spread of the latent rate $\theta$ to the intraclass correlation that drives the design effect. This $\rho$ is a correlation of \emph{correctness}; the concentration of the answer \emph{strings}, which governs selection, is a separate quantity, introduced in Section~\ref{sec:selection}. Figure~\ref{fig:correspondence} states the correspondence.

\begin{figure}[H]
\centering
\resizebox{\linewidth}{!}{%
\begin{tikzpicture}[
  font=\small,
  col/.style={text=cink, align=center, inner xsep=10pt, inner ysep=7pt,
              minimum height=11mm, minimum width=27mm},
  box/.style={draw=cnavy, fill=cpanel, line width=0.7pt, col},
  accent/.style={col, draw=cnavy, fill=cnavy, text=white, line width=0.7pt},
  head/.style={text=carrow, align=center, font=\small\bfseries},
  flow/.style={-{Stealth[length=2.6mm, width=2.0mm]}, line width=0.8pt,
               draw=carrow, shorten >=1.5pt, shorten <=1.5pt}
]
\node[head] (h1) at (0,1.45) {sampling frame};
\node[head] (h2) at (3.4,1.45) {cluster};
\node[head] (h3) at (6.8,1.45) {within-cluster draws};
\node[head] (h4) at (10.7,1.45) {correction};

\node[box] (s1) at (0,0) {population};
\node[box] (s2) at (3.4,0) {household};
\node[box] (s3) at (6.8,0) {$m$ members};
\node[accent] (s4) at (10.7,0) {$\deff = 1+(m{-}1)\rho$\\[2pt]$m \mapsto m/\deff$};
\draw[flow] (s1) -- (s2); \draw[flow] (s2) -- (s3); \draw[flow] (s3) -- (s4);

\node[box] (t1) at (0,-1.7) {benchmark};
\node[box] (t2) at (3.4,-1.7) {problem $q$};
\node[box] (t3) at (6.8,-1.7) {$n$ attempts $o_i$};
\node[accent] (t4) at (10.7,-1.7) {$\deff = 1+(n{-}1)\rho$\\[2pt]$n \mapsto \neff$};
\draw[flow] (t1) -- (t2); \draw[flow] (t2) -- (t3); \draw[flow] (t3) -- (t4);

\node[text=cink, font=\small\itshape, anchor=east] at (-1.7,0) {survey};
\node[text=cink, font=\small\itshape, anchor=east] at (-1.7,-1.7) {test-time};
\end{tikzpicture}}
\caption{The survey-to-test-time correspondence. \emph{Top:} a survey, where a household of $m$ members maps through $\deff = 1+(m-1)\rho$. \emph{Bottom:} test-time sampling, where a problem's $n$ attempts map through $\deff = 1+(n-1)\rho$ to $\neff$.}
\label{fig:correspondence}
\end{figure}
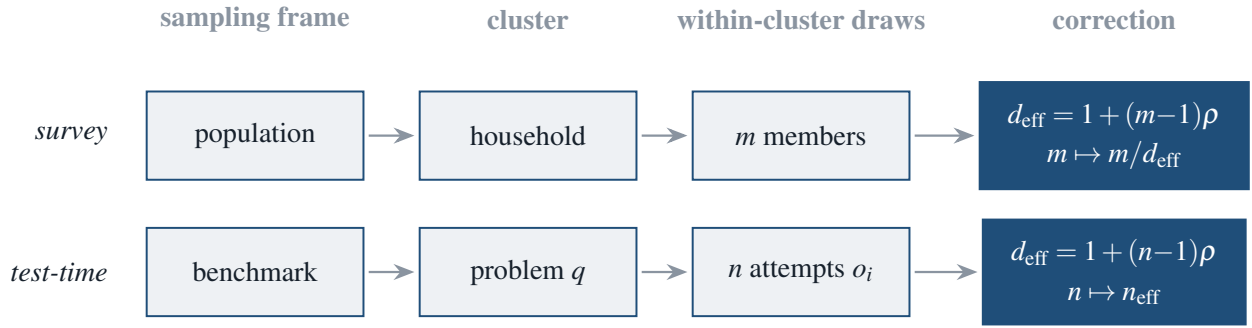

\subsection{Scope of the exact claims}
\label{sec:scope}

A canonical claim is only as strong as the assumptions it names, and the main one here is weaker than it looks. The load-bearing assumption is \emph{exchangeability}: the attempts to a problem are taken order-free, which gives them a common success rate $s$ and a common pairwise correlation $\rho$ and makes $\deff = 1+(n-1)\rho$ and the ceiling $1/\rho$ exact. Little of this rests on the correlation being the \emph{same} for every pair. The variance identity of Proposition~\ref{prop:deff} needs only the common rate $s$: for any pattern of pairwise correlations it holds with $\rho$ read as the \emph{mean} pairwise correlation $\bar\rho = \tfrac{1}{n(n-1)}\sum_{i\ne j}\Corr(Y_i,Y_j)$, so $\deff = 1+(n-1)\bar\rho$ and the ceiling is $1/\bar\rho$. Real sampling departs from equicorrelation (attempts that share a long prefix are more alike than attempts that branch early); then $1/\rho$ is simply read as $1/\bar\rho$, the average the estimator of Section~\ref{sec:estimate} already returns. What the ceiling truly requires is only that the correlation be positive and not vanish, that is, that sampling diversity stay bounded: the de Finetti mixture forces $\Var(\theta)\ge0$, hence $\rho\ge0$, and infinite exchangeability, the one further assumption, carries every $n\to\infty$ limit below. Equicorrelation is thus a convenience that turns an average into a single number, not a crutch the ceiling stands on.

Under these assumptions the design effect in Section~\ref{sec:neff} is \emph{exact} for one estimand: the success fraction $\hat p = K/n$, which depends on $(Y_i)$ only through their sum and so estimates the per-problem rate $s$ or, pooled, the benchmark mean. It governs how precisely a sampling budget pins down accuracy, the estimation use of test-time sampling. Two other quantities read the same draws but are not this mean and are treated on their own terms. Coverage, $\mathrm{pass}@n = \Prob[K\ge1]$, depends on the full mixing distribution $G$, and Section~\ref{sec:coverage} treats it directly through Equation~\eqref{eq:definetti}. Selection (self-consistency, best-of-$n$) returns the most-favored answer, the mode of the categorical answer distribution rather than a sum of the $Y_i$, and Section~\ref{sec:selection} gives it its own ceiling. The three are kept apart on purpose, because their different dependence on the same draws is the point here, and it is what lets one budget buy different amounts of each. The verifier is assumed accurate; an imperfect verifier adds a second, independent ceiling on usable accuracy, one not folded into $\rho$ here.

\section{The effective number of samples}
\label{sec:neff}

Everything downstream, the ceiling and the budget rule, follows from a single quantity: the variance of the number correct. This section computes it and reads off the effective number of samples.

\begin{proposition}[Design effect of test-time sampling]
\label{prop:deff}
Let $Y_1,\dots,Y_n$ be exchangeable with $\Prob[Y_i=1]=s$ and $\Corr(Y_i,Y_j)=\rho$ for $i\ne j$. Then the count $K=\sum_i Y_i$ has
\begin{equation}
\E[K] = ns,
\qquad
\Var(K) = ns(1-s)\,\big[\,1 + (n-1)\rho\,\big],
\label{eq:varK}
\end{equation}
and the success fraction $\hat p = K/n$ has $\Var(\hat p) = s(1-s)/\neff$ with the effective number of samples
\begin{equation}
\neff = \frac{n}{1 + (n-1)\rho}.
\label{eq:neff}
\end{equation}
\end{proposition}

\begin{proof}
Linearity gives $\E[K]=ns$. For the variance, $\Var(K) = \sum_i \Var(Y_i) + \sum_{i\ne j}\Cov(Y_i,Y_j)$. Each $\Var(Y_i)=s(1-s)$ and each of the $n(n-1)$ covariances equals $\rho\,s(1-s)$, so $\Var(K) = ns(1-s) + n(n-1)\rho\,s(1-s) = ns(1-s)[1+(n-1)\rho]$. Dividing by $n^2$ gives $\Var(\hat p) = s(1-s)[1+(n-1)\rho]/n = s(1-s)/\neff$.
\end{proof}

The bracket $\deff = 1+(n-1)\rho$ is the design effect: the factor by which correlation inflates the variance of the count relative to $n$ independent draws \citep{kish1965survey,cochran1977sampling}. Equivalently, the $n\times n$ covariance of the attempts is a constant diagonal plus a single rank-one term, $\Sigma = s(1-s)\big[(1-\rho)\mathbf{I} + \rho\,\mathbf{1}\mathbf{1}^\top\big]$, so the variance of the count collapses to the scalar $\deff$ without ever forming that matrix: the same structured-operator economy that lets large numerical systems be solved matrix-free \citep{bay2026no3d}. The estimator $\hat p$ behaves as if it were built from $\neff$ independent samples; Appendix~\ref{app:repro} confirms the variance identity~\eqref{eq:varK} against Monte Carlo simulation.

\subsection{The correlation ceiling}
\label{sec:ceiling}

This is the result this work is named for, and its single most important fact: the effective number of samples does not grow without bound. Equation~\eqref{eq:neff} has a finite limit.

\begin{corollary}[Correlation ceiling]
\label{cor:ceiling}
If $\rho > 0$, then $\neff$ increases in $n$ to the finite limit
\begin{equation}
\lim_{n\to\infty}\neff = \frac{1}{\rho},
\label{eq:ceiling}
\end{equation}
and $\neff$ reaches half of this ceiling at $n = (1-\rho)/\rho \approx 1/\rho$.
\end{corollary}

\begin{proof}
Write $\neff = n/[\,1-\rho + n\rho\,] \to 1/\rho$ as $n\to\infty$. Setting $\neff = 1/(2\rho)$ gives $2\rho n = 1+(n-1)\rho$, i.e.\ $\rho n = 1-\rho$, so $n = (1-\rho)/\rho$.
\end{proof}

This is the central fact. A correlated cluster with intraclass correlation $\rho$ is worth at most $1/\rho$ independent draws, however large the budget (exactly under the equicorrelation of Section~\ref{sec:scope}, and as a limiting average otherwise); and it gets halfway there by $n\approx1/\rho$. Figure~\ref{fig:effective_samples} plots the ceiling. A model with $\rho=0.1$ caps out near ten effective samples: the thousandth draw is almost worthless. The ceiling is also why the marginal value of a sample collapses.

\begin{figure}[H]
\centering
\includegraphics[width=0.9\linewidth]{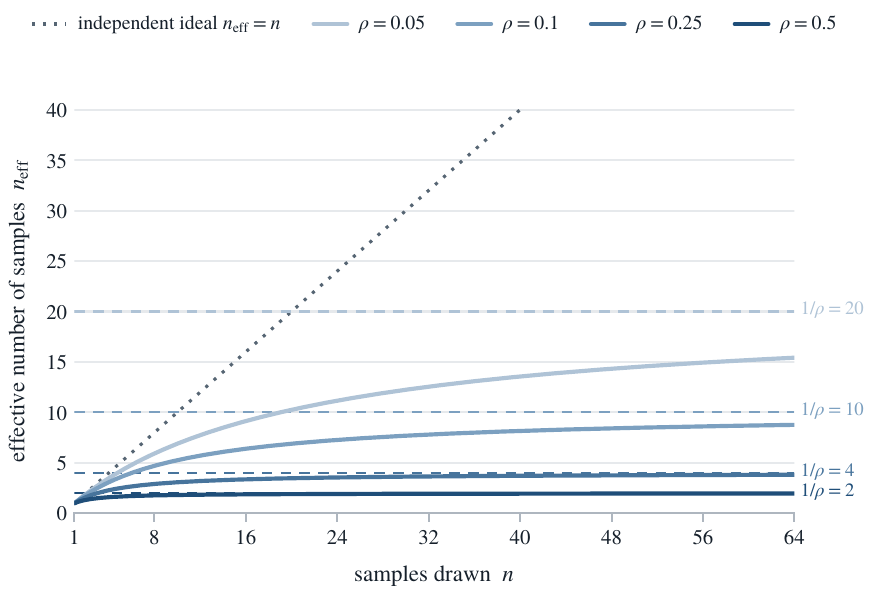}
\caption{The correlation ceiling. The effective number of samples $\neff = n/[1+(n-1)\rho]$ against the number drawn, for four correlation levels, each saturating at its ceiling $1/\rho$ (dashed); the independent ideal $\neff=n$ is the diagonal.}
\label{fig:effective_samples}
\end{figure}

\begin{corollary}[Value of the $n$-th sample]
\label{cor:marginal}
The marginal effective sample contributed by the $n$-th draw is
\begin{equation}
\frac{d\,\neff}{dn} = \frac{1-\rho}{\big[\,1+(n-1)\rho\,\big]^2}
\;\sim\; \frac{1-\rho}{\rho^2\,n^2}
\quad (n\to\infty),
\label{eq:marginal}
\end{equation}
so the worth of an added sample decays quadratically and is negligible once $n \gg 1/\rho$.
\end{corollary}

The marginal value starts at $d\,\neff/dn = 1-\rho$ at $n=1$ (the second draw already adds only $(1-\rho)/(1+\rho)$ effective samples) and then falls off as $1/(\rho n)^2$. Figure~\ref{fig:break_even} shows the collapse and marks the break-even point $n\approx1/\rho$ at which the curve has already given up most of its value. This is a budget rule, derived in Section~\ref{sec:allocate}: spending past $1/\rho$ samples buys redundancy, not signal.

A worked reading makes the decay concrete. At $\rho=0.1$ the first draw is worth a full independent sample; by Corollary~\ref{cor:marginal} the tenth is worth $(1-\rho)/[1+9\rho]^2 = 0.9/1.9^2 \approx 0.25$, and the hundredth only $0.9/10.9^2 \approx 0.008$. A thousand-sample run therefore carries the estimation information of roughly its first ten draws, and the rest refine the estimate by almost nothing. The decay is quadratic, so each tenfold increase in budget past the ceiling returns about a hundredth as much, which is why the curves of Figure~\ref{fig:break_even} fall to the axis so quickly.

\begin{figure}[H]
\centering
\includegraphics[width=0.9\linewidth]{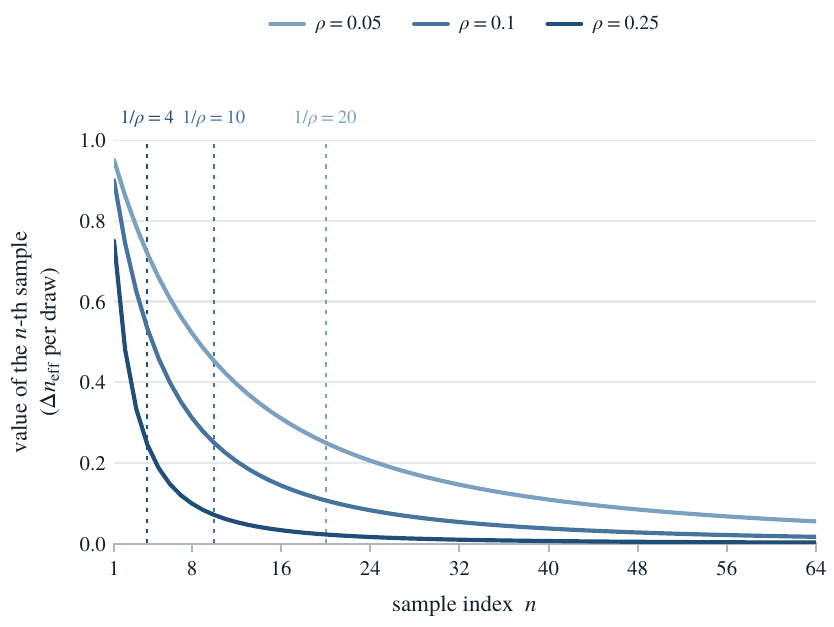}
\caption{The value of the $n$-th sample, $d\,\neff/dn$ from Corollary~\ref{cor:marginal}, against the sample index. Break-even points $n\approx1/\rho$ marked dotted (one per $\rho$).}
\label{fig:break_even}
\end{figure}

\section{Coverage rises, selection saturates}
\label{sec:cov_sel}

The most-cited puzzle in test-time scaling is that one curve keeps rising while another stalls, and no single mechanism connects them. This section gives the mechanism: coverage and selection read the same samples for different things, so they obey different ceilings. Coverage is capped only by what the model could ever produce; selection by whether its most common answer is correct. The gap between the two is exactly the solvable problems whose correct answer the vote fails to return.

Why selection matters at all bears stating, since coverage looks like the whole story. Coverage can be cashed only through a verifier that certifies a correct answer, and a perfect, general verifier is the rare exception rather than the rule. A few domains supply one: code against a test suite, a proof against a proof assistant, a numeric answer against a key. Most do not, and the verifiers that do exist score a proxy rather than truth, a finite test suite or a learned reward model that a wrong answer can still pass. Wherever no sound verifier is available, which is the typical deployment, a single answer can be returned only by selection, so the modal ceiling of Section~\ref{sec:selection}, not coverage, fixes the accuracy a system delivers. Coverage measures what the model could reach given a verifier it will not have; selection measures what ships.

\subsection{Coverage and the correlation tax}
\label{sec:coverage}

Coverage is the optimistic half of the split, the curve that keeps climbing; this subsection prices what correlation quietly subtracts from it. Coverage asks only whether some sample is correct, with no need to know which one. Under the mixture~\eqref{eq:definetti},
\begin{equation}
\mathrm{pass}@n = 1 - \Prob[K=0] = 1 - \E_\theta\big[(1-\theta)^n\big].
\label{eq:coverage}
\end{equation}
Coverage can only rise as the budget grows, and outside a hard core of unreachable problems it rises all the way.

\begin{proposition}[Coverage rises without a within-problem ceiling]
\label{prop:cov_mono}
Under the mixture~\eqref{eq:definetti}, $\mathrm{pass}@n$ is non-decreasing in $n$, strictly increasing whenever $G$ places mass on $(0,1)$, and converges to $1-\pi_0$, where $\pi_0=G(\{0\})$ is the fraction of problems the model never solves. For a fixed problem with reachability $p_q(c_q)=\pi>0$, coverage $1-(1-\pi)^n$ strictly increases to $1$.
\end{proposition}

\begin{proof}
For each $\theta\in[0,1]$ the map $n\mapsto(1-\theta)^n$ is non-increasing, strictly for $\theta\in(0,1)$. Taking the mixture expectation, $\Prob[K=0]=\E_\theta[(1-\theta)^n]$ is non-increasing in $n$, strictly if $G$ charges $(0,1)$, so $\mathrm{pass}@n=1-\Prob[K=0]$ is non-decreasing, strictly so. Since $(1-\theta)^n\to\mathbf 1\{\theta=0\}$ pointwise, dominated convergence gives $\Prob[K=0]\to\pi_0$, hence $\mathrm{pass}@n\to1-\pi_0$. The fixed-problem case is the point mass $\theta=\pi$.
\end{proof}

Correlation always taxes coverage relative to the independent reading.

\begin{proposition}[Correlation tax on coverage]
\label{prop:tax}
For exchangeable attempts with per-attempt success $s$,
\begin{equation}
\mathrm{pass}@n \;\le\; 1 - (1-s)^n,
\label{eq:tax}
\end{equation}
with equality if and only if $\rho = 0$. The gap grows with the dispersion of $\theta$: a mean-preserving spread of the difficulty distribution (more very-easy and very-hard problems at the same average) can only widen it.
\end{proposition}

\begin{proof}
The map $\theta \mapsto (1-\theta)^n$ is strictly convex on $[0,1]$ for $n\ge2$. By Jensen's inequality (the average of a convex function is at least the function of the average), $\E_\theta[(1-\theta)^n] \ge (1-\E[\theta])^n = (1-s)^n$, with equality iff $\theta$ is degenerate, i.e.\ $\rho=0$. Negating gives Equation~\eqref{eq:tax}.
\end{proof}

So the textbook curve $1-(1-s)^n$ is an upper bound, not a prediction: real coverage runs below it whenever samples are correlated. The shape of the shortfall is governed by the lower tail of $G$, the problems the model rarely solves.

\begin{proposition}[Power-law coverage from heterogeneity]
\label{prop:powerlaw}
Let $\theta \sim \mathrm{Beta}(\alpha,\beta)$ (a flexible family of success-rate distributions on the unit interval, standing in for the spread of problem difficulty), so $s = \alpha/(\alpha+\beta)$ and $\rho = 1/(\alpha+\beta+1)$. Then the miss rate is
\begin{equation}
\Prob[K=0] = \frac{B(\alpha, \beta+n)}{B(\alpha,\beta)}
\;\sim\; \frac{\Gamma(\alpha+\beta)}{\Gamma(\beta)}\, n^{-\alpha}
\qquad (n\to\infty),
\label{eq:powerlaw}
\end{equation}
so coverage approaches its limit as a power law $n^{-\alpha}$, not the exponential $(1-s)^n$.
\end{proposition}

\begin{proof}
With $\theta\sim\mathrm{Beta}(\alpha,\beta)$, $\E[(1-\theta)^n] = B(\alpha,\beta+n)/B(\alpha,\beta)$ is the standard beta-binomial (a binomial whose success rate is itself drawn from the Beta) probability of zero successes. Writing it as $\frac{\Gamma(\alpha+\beta)}{\Gamma(\beta)}\cdot\frac{\Gamma(\beta+n)}{\Gamma(\alpha+\beta+n)}$ and using $\Gamma(\beta+n)/\Gamma(\alpha+\beta+n)\sim n^{-\alpha}$ (Stirling) gives the tail. The intraclass correlation of a beta-binomial is $\rho = 1/(\alpha+\beta+1)$.
\end{proof}

Proposition~\ref{prop:powerlaw} is the difficulty-heterogeneity story of \citet{schaeffer2025powerlaws} and \citet{kazdan2025predicting}, recovered here as the coverage face of the same model for a Beta difficulty prior: a heavy lower tail (small $\alpha$) yields the slowly rising, log-linear coverage that \citet{brown2024monkeys} fit empirically, with the exponent set by the lower tail of $G$. Figure~\ref{fig:coverage_powerlaw} contrasts the exponential and power-law regimes. The overdispersed binomial behind it (a count with more spread than independent draws would give) is classical \citep{skellam1948probability}. If $G$ additionally places an atom $\pi_0$ at $\theta=0$, a hard core of attempts the model never makes correct, then $\mathrm{pass}@n \to 1-\pi_0 < 1$ and no budget crosses it; this is mode collapse in its starkest form. The atom is a ceiling of model capability rather than of sampling: those problems lie outside the model's reach, a generalization limit that no sampling budget repairs \citep{bay2024generalization}.

The practical force of the power law is how slowly it pays. An exponential miss rate $(1-s)^n$ clears any target in a handful of samples, each one cutting the miss rate by a constant factor; a power-law miss rate $n^{-\alpha}$ does not. Halving an exponential miss rate costs a fixed number of further samples; halving a power law costs a fixed \emph{multiplicative} factor $2^{1/\alpha}$, so for a heavy lower tail (small $\alpha$) each successive halving costs as much sampling as everything before it combined. Coverage keeps rising, but with sharply diminishing speed: the slow, log-linear climb that the released logs show over four orders of magnitude. Figure~\ref{fig:coverage_powerlaw} plots both regimes on log--log axes, where the exponential falls off a cliff and the power law settles onto a straight line of slope $-\alpha$.

\begin{figure}[H]
\centering
\includegraphics[width=0.9\linewidth]{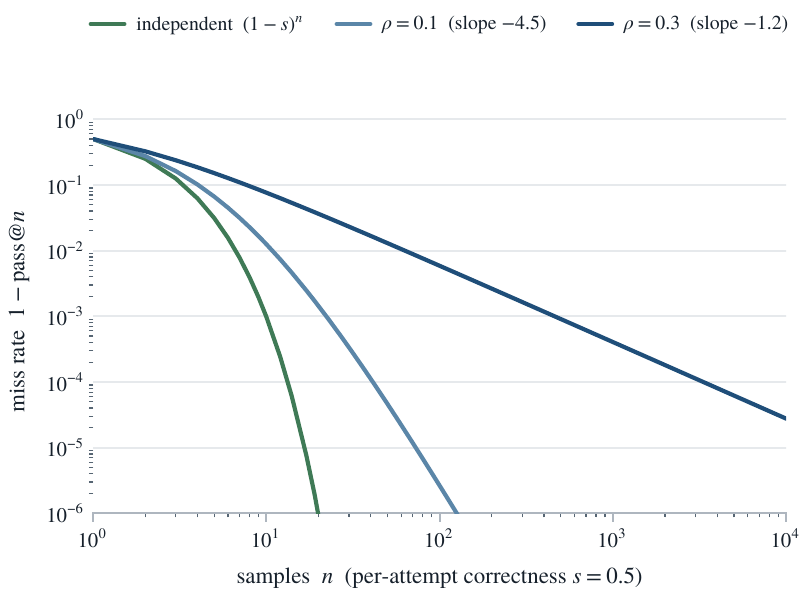}
\caption{Exponential versus power-law coverage. The miss rate $1-\mathrm{pass}@n$ on log--log axes: the independent exponential against the power law $n^{-\alpha}$ of Proposition~\ref{prop:powerlaw}. Here $s=0.5$ and $\rho\in\{0.1,0.3\}$.}
\label{fig:coverage_powerlaw}
\end{figure}

\subsection{Selection and the modal answer}
\label{sec:selection}

Selection is where the visible gains stop, and the cause is not correlation but the shape of the answer distribution. To return an answer without a verifier, a method keeps the most common answer (self-consistency \citep{wang2023selfconsistency}) or the highest-scored one (reward-model best-of-$n$); either way it reads the whole distribution over answer strings, not just the per-attempt correctness rate. Fix a problem $q$ and let the model induce a distribution $p_q(a)$ over candidate answers $a$, with correct answer $c_q$ and a unique most common answer, the \emph{mode} $a^\star_q = \arg\max_a p_q(a)$. A plurality vote among $n$ attempts returns the empirical mode.

\begin{proposition}[Selection ceiling]
\label{prop:vote}
As $n\to\infty$ the plurality answer converges almost surely to the mode $a^\star_q$, so self-consistency accuracy on a fixed problem tends to $\mathbf 1\{a^\star_q = c_q\}$. Averaged over a benchmark, self-consistency accuracy converges to the \emph{modal-hit rate}
\begin{equation}
\pi_{\mathrm{mode}} \;=\; \Prob_q\big[\,a^\star_q = c_q\,\big],
\label{eq:vote}
\end{equation}
a hard ceiling, the \emph{modal ceiling}, with no dependence on the sample budget $n$.
\end{proposition}

\begin{proof}
For a fixed problem the empirical answer frequencies converge to $p_q$ almost surely (the law of large numbers applied to the multinomial counts), and $\arg\max$ is continuous wherever the maximizer is unique, so the empirical mode converges to $a^\star_q$ almost surely. Hence $\mathbf 1\{\text{plurality correct}\}\to\mathbf 1\{a^\star_q=c_q\}$, and averaging over problems gives the limit~\eqref{eq:vote}, which is constant in $n$.
\end{proof}

The ceiling is a wall, not a slowdown: once the budget reveals the mode, more samples cannot move the vote, and on the problems where the mode is wrong they move it the wrong way.

\begin{corollary}[Anti-scaling of selection]
\label{cor:antiscale}
On any problem whose mode is incorrect ($a^\star_q \ne c_q$) yet whose correct answer is reachable ($p_q(c_q)>0$), self-consistency accuracy falls to $0$ while coverage rises to $1$ as $n\to\infty$: more sampling makes selection worse and coverage better on the same problem.
\end{corollary}

\begin{proof}
By Proposition~\ref{prop:vote} the plurality converges to $a^\star_q\ne c_q$, so $\mathbf 1\{\text{plurality correct}\}\to0$, while $\mathrm{pass}@n = 1-(1-p_q(c_q))^n\to1$ because $p_q(c_q)>0$.
\end{proof}

Anti-scaling is the mechanism behind the limited and sometimes non-monotone returns that self-consistency and reward-model selection show in practice \citep{chen2024morecalls,wang2025diversity}: extra samples sharpen a confident wrong answer. How fast a problem reaches its ceiling is set by the concentration of $p_q$, summarized by the \emph{effective number of answers} $1/\sum_a p_q(a)^2$; the ceiling \emph{height} $\pi_{\mathrm{mode}}$ is set by whether the mode it converges to happens to be correct. Figure~\ref{fig:anti_scaling} shows both outcomes: coverage rises to one whether or not the mode is correct, while selection rises to one only when the mode is correct and otherwise falls to zero. Self-consistency and plurality voting meet this modal ceiling exactly. Best-of-$n$ with a learned reward model is bounded instead by how well that scorer ranks answers: a perfect verifier lifts it to coverage, a frequency-like scorer collapses it to the modal ceiling, and real reward models fall between.

\begin{figure}[H]
\centering
\includegraphics[width=0.9\linewidth]{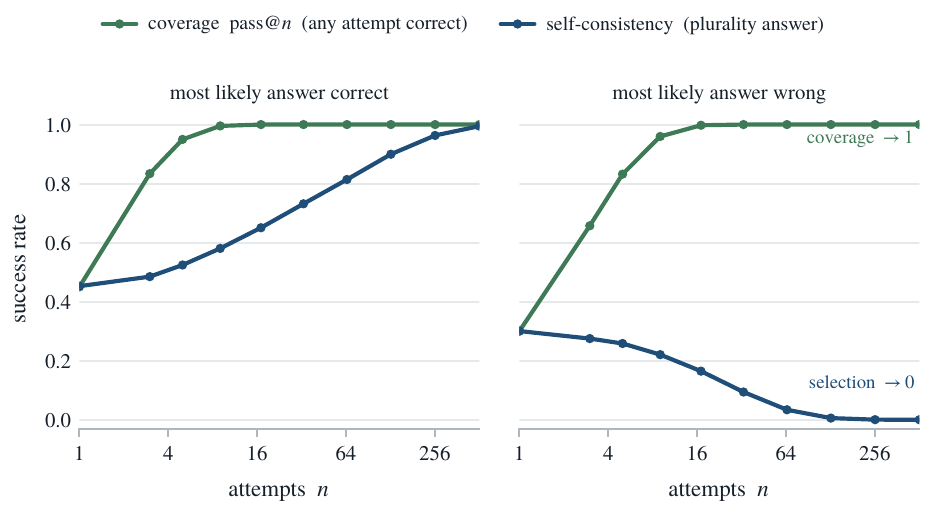}
\caption{Anti-scaling of selection. Two problems sampled to $n$ attempts each: coverage (any attempt correct) and self-consistency (the plurality answer) against $n$. Left: a problem whose most common answer is correct. Right: a problem whose most common answer is wrong.}
\label{fig:anti_scaling}
\end{figure}

The familiar majority-vote reading is the binary special case, and it is a strict lower bound. Collapse the answers to two, the correct one against a single modal error, and let $\theta$ be the latent per-attempt correctness of Equation~\eqref{eq:definetti}; then plurality is majority, $K/n\to\theta$ almost surely, and the vote is correct in the limit exactly when $\theta>\tfrac12$.

\begin{corollary}[Condorcet jury, correlated]
\label{cor:condorcet}
If the difficulty distribution places positive mass on problems the model gets wrong more often than right ($\Prob[\theta<\tfrac12]>0$), majority-vote accuracy converges to $\Prob[\theta>\tfrac12]<1$, below the value $1$ that independent jurors reach when $s>\tfrac12$ \citep{ladha1992condorcet,boland1989majority}; the normal approximation $\theta\approx\mathcal N(s,\rho\,s(1-s))$ gives $\Prob[\theta>\tfrac12]\approx\Phi\big((s-\tfrac12)/\sqrt{\rho\,s(1-s)}\big)$. This never exceeds the plurality ceiling, $\Prob[\theta>\tfrac12]\le\pi_{\mathrm{mode}}$, because $\theta>\tfrac12$ forces every wrong answer below $\tfrac12$ and so below the correct one.
\end{corollary}

The two bounds part company on real logs, and the distance between them is the point. On the dependent-draw log of Section~\ref{sec:gap} the majority bound reads $\Prob[\theta>\tfrac12]=0.20$ while the plurality ceiling is $\pi_{\mathrm{mode}}=0.45$: more than half of the selectable accuracy comes from problems the model answers correctly less than half the time, where the correct answer is still the single most common one because the errors scatter. Figure~\ref{fig:majority_vote} shows the binary plateau and that simulated accuracy lands on it.

\begin{figure}[H]
\centering
\includegraphics[width=0.9\linewidth]{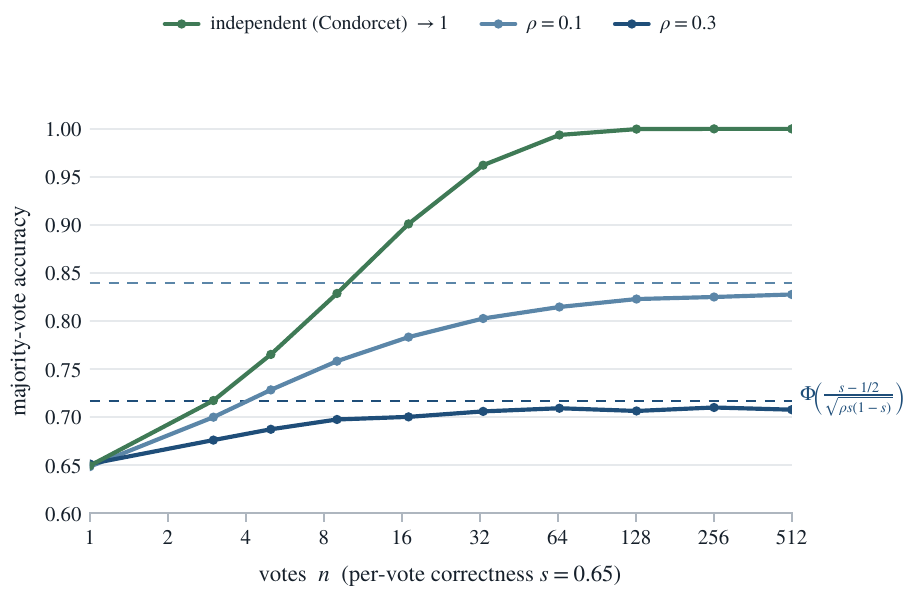}
\caption{The binary (majority-vote) special case of the selection ceiling. Majority-vote accuracy against $n$ for several $\rho$: the independent case ($\rho\to0$, top) and correlated cases plateauing at $\Prob[\theta>\tfrac12]$ (Corollary~\ref{cor:condorcet}), with its normal approximation dashed. Markers: simulation; $s=0.65$.}
\label{fig:majority_vote}
\end{figure}

\subsection{The coverage--selection gap, measured}
\label{sec:gap}

Theory has predicted a gap; this is where it is measured on real sampling logs. Put the two together. Coverage (Proposition~\ref{prop:tax}) climbs toward what the model can reach, $1-\pi_0$, slowly. Selection (Proposition~\ref{prop:vote}) saturates at the modal-hit rate $\pi_{\mathrm{mode}}$. The two diverge, and the divergence is the identifiability gap: the problems whose correct answer is in the pool but is not its most common answer. The split shows up in two released logs that probe the two effects separately: an independent-draw log for the between-problem difficulty spread, and a dependent-draw log for the within-problem answer collapse.

\textbf{The between-problem term, on independent draws.} \citet{brown2024monkeys} sampled up to $10^4$ solutions per problem on GSM8K \citep{cobbe2021gsm8k} (grade-school math word problems) and MATH \citep{hendrycks2021math} (harder competition mathematics) and recorded the correctness of each. Their attempts are drawn independently, so the within-problem correlation is zero by construction; what the logs expose is the \emph{between-problem} difficulty spread $\rho_b$ and its consequences: the $\rho_w=0$ face of the model. The analysis here estimates, per configuration, the difficulty mean $s$, the intraclass correlation $\hat\rho = \Var(\theta)/[s(1-s)]$, coverage by the unbiased estimator of \citet{chen2021codex}, and self-consistency by plurality vote of the extracted answers (validated against the logs' own correctness labels). The self-consistency plateau is reported only where the numeric extractor reproduces those labels; MATH answers are boxed expressions a numeric extractor cannot parse, so MATH shows coverage and $\hat\rho$ alone. Table~\ref{tab:empirical} reports them.

\begin{table}[H]
\centering
\caption{The between-problem difficulty correlation $\hat\rho_b$ and the coverage--selection gap, from the released logs of \citet{brown2024monkeys} ($10^4$ samples per problem). The $\hat\rho_b$ $95\%$ CI is a problem-level clustered bootstrap ($10^4$ resamples). Self-consistency is n/a on MATH (boxed answers a numeric extractor cannot parse).}
\label{tab:empirical}
{\small
\renewcommand{\arraystretch}{1.2}
\begin{tabular}{@{}llcccccc@{}}
\toprule
\textbf{Benchmark} & \textbf{Model} & $s$ & $\hat\rho_b$ & $\hat\rho_b$ 95\% CI & $1/\hat\rho_b$ & coverage@$10^4$ & self-cons. \\
\midrule
GSM8K & Llama-3-8B-Instruct  & 0.77 & 0.47 & $[0.40,\,0.54]$ & 2.1 & 1.00 & 0.87 \\
GSM8K & Llama-3-70B-Instruct & 0.93 & 0.41 & $[0.24,\,0.54]$ & 2.5 & 1.00 & 0.97 \\
MATH  & Llama-3-8B-Instruct  & 0.27 & 0.48 & $[0.41,\,0.55]$ & 2.1 & 0.98 & n/a \\
MATH  & Llama-3-70B-Instruct & 0.48 & 0.61 & $[0.54,\,0.66]$ & 1.6 & 0.98 & n/a \\
\bottomrule
\end{tabular}}
\end{table}

Two readings stand out. First, the difficulty correlation is large and stable, $\hat\rho\approx0.4$--$0.6$ across models and benchmarks, so by Equation~\eqref{eq:neff} the $10^4$ samples drawn for each problem carry the benchmark-mean information of only $\neff\approx2$ independent samples: for estimating benchmark accuracy, ten thousand samples of one problem are worth about two independent ones. Second, coverage and selection diverge in the direction the split predicts. Figure~\ref{fig:empirical_gap} shows the GSM8K, Llama-3-8B-Instruct curves: coverage reaches $1.00$ while self-consistency plateaus at $0.87$. For roughly an eighth of problems the correct answer is somewhere in the pool but the vote does not return it even at the plateau: the usable-signal gap, on real data.

\begin{figure}[H]
\centering
\includegraphics[width=0.9\linewidth]{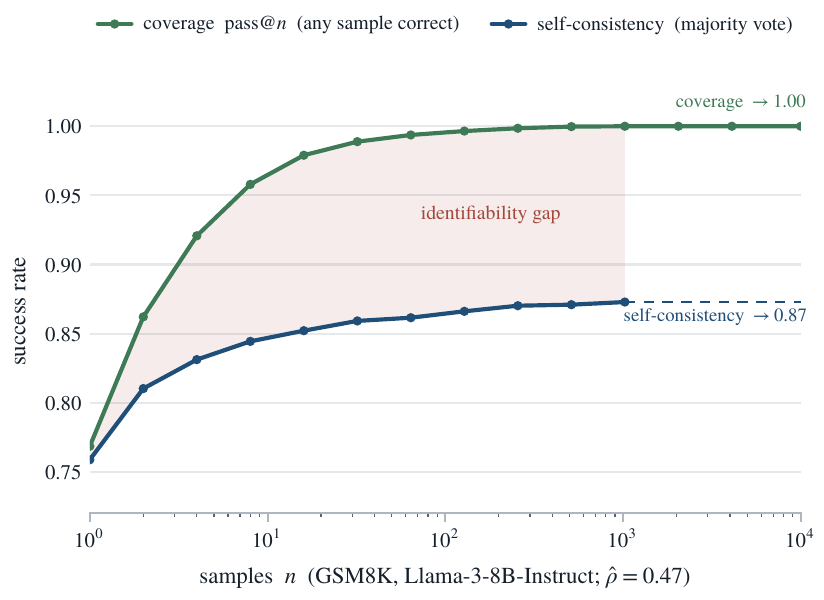}
\caption{The coverage--selection gap on real logs \citep{brown2024monkeys}: GSM8K, Llama-3-8B-Instruct, up to $10^4$ samples per problem. Coverage rises to $1.00$; self-consistency plateaus at $0.87$; the shaded band is the usable-signal gap.}
\label{fig:empirical_gap}
\end{figure}

The reading is consistent across the repeated-sampling literature. \citet{chen2024morecalls} find voting performance non-monotone in the number of calls because a benchmark mixes easy and hard problems, the anti-scaling of Corollary~\ref{cor:antiscale}; \citet{wang2025diversity} show that prompt-diversity interventions, which spread the answer distribution, are what move best-of-$n$; and \citet{kirk2024diversity} document that post-training (the fine-tuning applied after pretraining) sharpens the modal answer. Each is a face of one split: selection is held at its modal ceiling, coverage is not.

\textbf{The selection ceiling, on raw answers.} The modal-hit rate $\pi_{\mathrm{mode}}$ needs the answer strings, which the correctness-only logs above do not record; a log of raw completions supplies them. The best-of-$n$ release of \citet{beeching2024scaling} samples each of the $500$ MATH-500 problems $256$ times from one model (Llama-3.2-1B-Instruct) at a fixed decoding configuration (temperature $0.8$, top-$p$ $1.0$), recording every raw completion. Across the $500$ problems the $256$ answers carry a median of only about thirteen distinct values (an effective answer count $1/\sum_a p_a^2$), not $256$: the answer distribution $p_q$ is sharply concentrated. Figure~\ref{fig:within_session} measures the consequence: coverage (any attempt correct) climbs to $0.88$ while self-consistency (the plurality answer) plateaus at $\pi_{\mathrm{mode}}=0.45$ by about $n=64$ and moves little after. Correctness is graded with the same verifier that produced the dataset's labels (\texttt{math-verify}), reproducing its reported single-sample accuracy ($\approx0.27$) to about a point. The binary majority bound on the same log reads only $\Prob[\theta>\tfrac12]=0.20$: half the realized selection accuracy comes from problems solved less than half the time whose scattered errors leave the correct answer as the single most common one. The plateau is the modal ceiling of Proposition~\ref{prop:vote}, on real attempts: the correct answer reaches the pool for nearly nine problems in ten, but the most common answer is correct for only four and a half.

\begin{figure}[H]
\centering
\includegraphics[width=0.9\linewidth]{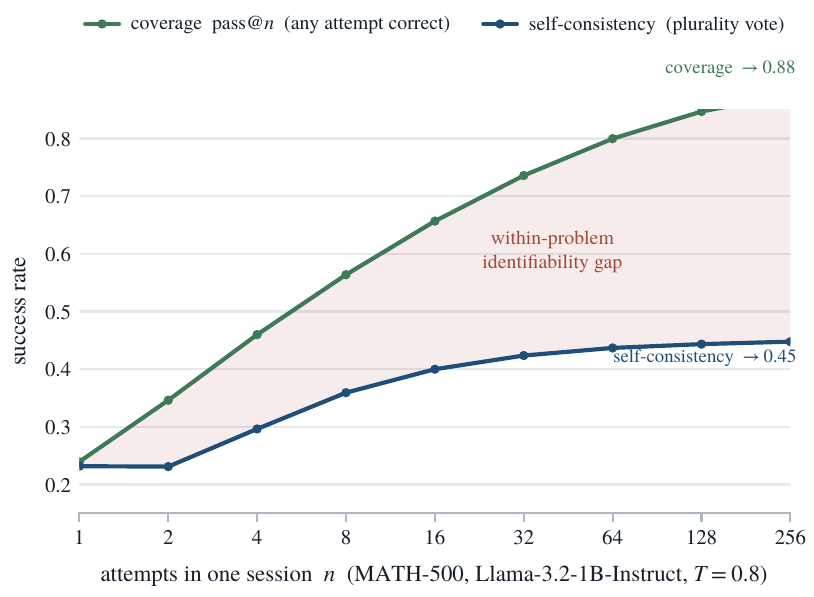}
\caption{The within-problem coverage--selection gap on a dependent-draw log \citep{beeching2024scaling}: MATH-500, Llama-3.2-1B-Instruct, $256$ attempts per problem in one session (averaged over five sessions). Coverage rises to $0.88$; self-consistency plateaus at $0.45$; the shaded band is the within-problem usable-signal gap.}
\label{fig:within_session}
\end{figure}

The five sessions also test the decomposition itself. Estimating $\rho_w$ as the run-to-run dispersion of a problem's success rate across the five independent sessions, corrected for the binomial noise of a finite session, gives $\hat\rho_w$ near zero ($\hat\rho_w \approx 0.0007$, $95\%$ CI $[0.0005,\,0.0009]$): re-running a fixed model at a fixed temperature does not move the latent success rate beyond sampling noise, so the de Finetti rate $\theta$ is nearly fixed per problem and the same-session pooled correlation is dominated by difficulty. The two-stage identity then holds on real data to within $0.001$: the directly pooled same-session correlation is $\hat\rho=0.401$, against $\hat\rho_b + (1-\hat\rho_b)\hat\rho_w = 0.402$ from the separate terms (Equation~\eqref{eq:decomp}, with this log's own $\hat\rho_b=0.402$). The selection plateau in Figure~\ref{fig:within_session} therefore is not a between-session effect: it is the within-session answer collapse, the concentration of the answer distribution $p_q$ rather than any drift in the success rate, which Proposition~\ref{prop:vote} prices as the modal-hit rate $\pi_{\mathrm{mode}}$ and Corollary~\ref{cor:condorcet} lower-bounds by $\Prob[\theta>\tfrac12]$.

\begin{takeaway}
Repeated sampling answers two questions that a single accuracy number blurs. \emph{Coverage} asks whether a correct answer is reachable at all, and keeps rising with the budget. \emph{Selection} asks whether the model's most common answer is the correct one, and is capped at the modal-hit rate $\pi_{\mathrm{mode}}$. Their difference is the \emph{identifiability gap}: the problems a model can reach but not return.
\end{takeaway}

\section{Decomposition, allocation, and measurement}
\label{sec:practice}

A lens earns its place only if it can be measured and acted on. This section splits the correlation into its two sources, turns the ceiling into a compute-allocation rule, and gives a one-line estimator for $\rho$ that runs on any sampling log already in hand.

\subsection{Within- and between-problem correlation}
\label{sec:decomp}

Two sources feed the pooled correlation $\rho$, and separating them is what tells a benchmark mean from a single answer. \emph{Within-problem} correlation $\rho_w$ is run-to-run dependence on a fixed problem: the dispersion of the latent rate $\theta$ across independent sessions in Equation~\eqref{eq:icc}. Under standard independent decoding this is zero by construction, each attempt an independent draw given the prompt, and Section~\ref{sec:gap} measures it at $\hat\rho_w\approx0.0007$. \emph{Between-problem} heterogeneity $\rho_b$ is the spread of difficulty across a benchmark: two attempts to a randomly drawn problem move together merely because they share that problem's difficulty. For attempts pooled over a benchmark the variance components add, so the pooled intraclass correlation (taking $\rho_w$ common across problems) satisfies
\begin{equation}
\rho \;=\; \rho_b + (1-\rho_b)\,\rho_w ,
\label{eq:decomp}
\end{equation}
the standard two-stage design effect of clustered sampling \citep{cochran1977sampling}: the total same-problem correlation is the between-problem difficulty spread plus the within-problem dependence left after it. With $\rho_w\approx0$ the pooled $\rho$ is essentially $\rho_b$, so the correlation ceiling of Section~\ref{sec:neff} is a statement about \emph{estimating} a benchmark from heterogeneous problems, not about repeated tries on one problem. The slow, power-law coverage of a benchmark is the same $\rho_b$ through Proposition~\ref{prop:powerlaw}, the regime studied by \citet{schaeffer2025powerlaws}. The selection ceiling is a third axis, the concentration of the answer distribution $p_q$ (Section~\ref{sec:selection}), which $\rho$ does not see: difficulty heterogeneity, run-to-run dependence, and answer-mode collapse are three distinct quantities the nominal sample count silently confounds.

\subsection{A compute-allocation rule}
\label{sec:allocate}

This is the practical payoff: with a fixed compute budget, the goal decides how to spend it, because the three goals reach their ceilings at different points. Let total inference compute be $C \approx n L$ for $n$ samples of reasoning length $L$. The levers are more samples (raise $n$), longer reasoning (raise $L$, which raises $s$), more problems (for evaluation), and decorrelation (spread the answer distribution via temperature, nucleus (top-$p$) sampling, prompt diversity, or mixing models) \citep{holtzman2020curious,wang2025diversity,muennighoff2025s1}. The ceilings of Sections~\ref{sec:neff} and~\ref{sec:selection} (Corollaries~\ref{cor:ceiling}--\ref{cor:marginal} and Proposition~\ref{prop:vote}) set the stopping point for each:

\begin{takeaway}
To estimate a benchmark mean, $n \approx 1/\rho_b$ samples per problem suffice (about two at the measured $\rho_b$); spend the rest on more problems. To select an answer without a verifier, sampling helps only until the plurality stabilizes, on the order of the effective answer count $1/\sum_a p_q(a)^2$, and past that it can anti-scale. To cover, that is to find a correct sample for a verifier, there is no within-problem ceiling and more samples keep paying.
\end{takeaway}

The rule reframes the ``think longer or sample more'' question \citep{snell2024scaling,wu2024inference,liu2025surpass}: sampling pays without limit for coverage, stops early for selection, and stops almost at once for estimation, so the budget should follow the goal rather than a single number. Whether decorrelation actively raises the selection plateau is the lever the protocol below tests.

\subsection{Estimating $\rho$}
\label{sec:estimate}

The ceiling is actionable only if $\rho$ can be read off runs already in hand, and it can. The correlation is measurable from any sampling log that records, for each of $M$ problems, the number correct $c_i$ out of $n_i$ attempts. The standard moment (analysis-of-variance) estimator for clustered binary data \citep{cochran1977sampling} compares the between- and within-problem sums of squares of $\hat p_i = c_i/n_i$,
\begin{equation}
\hat\rho \;=\; \frac{\mathrm{MS}_{\text{between}} - \mathrm{MS}_{\text{within}}}
{\mathrm{MS}_{\text{between}} + (n_0-1)\,\mathrm{MS}_{\text{within}}},
\label{eq:icc_est}
\end{equation}
with $n_0$ the average cluster size; equivalently, the beta-binomial model of \citet{kazdan2025predicting} has intraclass correlation $\hat\rho = 1/(\hat\alpha+\hat\beta+1)$ by the identity in Appendix~\ref{app:derivations}. Two cautions. On benchmarks with many fully-solved or never-solved problems the within-problem variance is degenerate and Equation~\eqref{eq:icc_est} can return a small or negative $\hat\rho$, which should be clipped at zero; and a pooled $\hat\rho$ mixes $\rho_b$ and $\rho_w$, so it states how much of the nominal $n$ is real only at the level (a single problem or a whole benchmark) at which it was estimated. Table~\ref{tab:empirical} runs the estimator on the logs of \citet{brown2024monkeys}; because those attempts are drawn independently, what it recovers is the between-problem term $\rho_b$, with a clustered (problem-level) bootstrap interval that resamples problems with replacement. The within-problem term needs dependent draws, which Section~\ref{sec:gap} reads off the single-session log of \citet{beeching2024scaling}: the seed-to-seed spread of a problem's success rate is negligible there ($\hat\rho_w\approx0.0007$), so re-decoding one prompt barely moves the latent rate, yet within a session the \emph{answer} distribution still narrows to about thirteen modes and self-consistency plateaus at $0.45$. The single-model selection ceiling is therefore set not by run-to-run drift but by the concentration of the answer distribution, the modal-hit rate $\pi_{\mathrm{mode}}$ of Proposition~\ref{prop:vote}. Genuine correlation does cap voting once the voters are distinct models rather than repeated draws of one: across nine distinct judge models the mean pairwise error correlation is $\hat\rho\approx0.39$, so the nine are worth only about two effective votes \citep{kohli2026judges}. For repeated draws of a single model, by contrast, a competition report finds high-temperature sampling already largely decorrelates the errors \citep{nitarach2026aimo}, so what holds selection down there is not run-to-run correlation but the concentration of the answer mode: the handful of effective answers that makes selection saturate so early.

The practical recommendation is one line, meant to be copied into a methods section verbatim:
\begin{takeaway}
Estimate $\hat\rho$ from the sampling log via Equation~\eqref{eq:icc_est}, then report the \emph{effective number of samples} $\neff = n/[1+(n-1)\hat\rho]$ alongside the nominal count $n$, with the ceiling $1/\hat\rho$.
\end{takeaway}
This is the design-effect-corrected count of \citet{kish1965survey}, and it lets a reader see how much of a sampling budget is real.

\paragraph{A pre-registered protocol for the decoding lever.}
Section~\ref{sec:gap} measures the within-problem ceiling at one decoding configuration: the latent run rate $\theta$ is nearly fixed per problem ($\hat\rho_w\approx0$), so the selection plateau is set by the dispersion of the modal answer within a session, not by run-to-run drift. What remains to be measured is the \emph{lever}: whether the decoding choices that decorrelate answers raise the plateau, the prediction the allocation rule rests on. The following test is pre-registered. Fix a benchmark and a model. For each of $M\ge200$ problems and each of several decoding configurations (a sweep over temperature, nucleus $p$, and prompt diversity), draw $m\ge256$ verified attempts; measure the within-session answer-effective count $1/\sum_a p_a^2$, the plurality plateau, and the answer-indicator intraclass correlation. The pre-registered prediction, from Proposition~\ref{prop:vote}, is that configurations with a lower answer correlation have a higher plurality plateau, monotone in the effective answer count; the falsifiable alternative is that the plateau is flat in the decoding configuration, under which decorrelation would not be the lever the rule names. This requires model inference and is left to follow-on work; it is stated now so the decoding lever is testable rather than assumed.

\subsection{A practitioner's reference}
\label{sec:summary}

Everything needed to apply the lens fits on a page: a short checklist of what to do, a table of the formulas, and a reading of the common methods. The checklist comes first, and these are diagnostics, not universal defaults.
\begin{enumerate}
\item For evaluation, report $\neff$, not just $n$: for the benchmark mean a budget of $n$ per problem is worth $n/[1+(n-1)\rho_b]$ independent observations, at most $1/\rho_b$, so add problems rather than samples.
\item Estimate $\rho_b$ from the log with Equation~\eqref{eq:icc_est}, clipped at zero on saturated benchmarks.
\item Distinguish the three goals: coverage keeps improving with $n$, selection plateaus at $\pi_{\mathrm{mode}}$ and can anti-scale, estimation saturates by $n\approx1/\rho_b$.
\item For selection, lower the answer concentration (temperature, nucleus sampling, prompt diversity, model mixing) rather than raising $n$; sampling past the effective answer count buys nothing.
\item Separate the answer ceiling from the verifier ceiling: one is whether the correct answer is the mode, the other whether the right try is identifiable by the scorer.
\end{enumerate}

A concrete pass through the checklist fixes the workflow. Take a benchmark on which the difficulty correlation is $\hat\rho_b\approx0.5$ and the answers concentrate onto about a dozen modes per problem, the regime of the logs in Section~\ref{sec:gap}. For evaluation, a budget of $256$ samples per problem is worth only $\neff\approx1/\rho_b\approx2$ independent observations of the benchmark mean, so a confidence interval computed as if the $256$ were independent is too narrow by roughly elevenfold ($\sqrt{256/2}$); the honest move is to report $\neff$ and to spend fresh compute on fresh problems. For deployment with a verifier, the same $256$ samples keep paying, since coverage carries no within-problem ceiling. For deployment without one, self-consistency has captured almost all it will within a few dozen samples, a small multiple of the effective answer count, and drawing the full $256$ risks anti-scaling on the problems whose mode is wrong.

Three numbers drive the three goals: the difficulty correlation $\rho_b$, the effective answer count, and the verifier's accuracy. The first two are read off any sampling log already in hand by the estimators above, and the third is a property of the scorer. Reading them once states how much of a budget is real for each use, which is the whole of what the lens asks of a practitioner. Table~\ref{tab:formulas} collects the formulas behind the checklist, each with the behavior it explains, from the design effect that opens the argument to the modal-hit rate that closes it.

\begin{table}[H]
\centering
\caption{Formulas for correlated test-time scaling. Here $n$ is the number of samples, $s$ the per-attempt success probability, and $\rho$ the intraclass correlation of the success indicators.}
\label{tab:formulas}
{\small
\renewcommand{\arraystretch}{1.22}
\begin{tabular}{@{}p{0.27\linewidth}p{0.37\linewidth}p{0.28\linewidth}@{}}
\toprule
\textbf{Quantity} & \textbf{Formula} & \textbf{Interpretation} \\
\midrule
Design effect & $\deff = 1+(n-1)\rho$ & Correlation inflates the variance of the count. \\
\addlinespace[2pt]
Effective number of samples & $\neff = n/[1+(n-1)\rho]$ & The usable count of independent draws. \\
\addlinespace[2pt]
Correlation ceiling & $\neff \to 1/\rho$ & Sampling cannot exceed $1/\rho$ effective draws. \\
\addlinespace[2pt]
Value of the $n$-th sample & $(1-\rho)/[1+(n-1)\rho]^2$ & Marginal worth decays as $1/(\rho n)^2$. \\
\addlinespace[2pt]
Coverage tax & $\mathrm{pass}@n \le 1-(1-s)^n$ & Correlation runs coverage below the independent curve. \\
\addlinespace[2pt]
Coverage tail & $\Prob[K{=}0]\sim \tfrac{\Gamma(\alpha+\beta)}{\Gamma(\beta)}\,n^{-\alpha}$ & Heterogeneity gives power-law, not exponential, coverage. \\
\addlinespace[2pt]
Selection ceiling & $\pi_{\mathrm{mode}} = \Prob_q[\,a^\star_q\text{ correct}\,]$ & Self-consistency cannot beat the modal-hit rate. \\
\addlinespace[2pt]
Effective number of answers & $1/\sum_a p_q(a)^2$ & Sets how fast plurality reaches its ceiling. \\
\addlinespace[2pt]
Majority-vote bound & $\Prob[\theta>\tfrac12]\approx\Phi\big(\tfrac{s-1/2}{\sqrt{\rho s(1-s)}}\big)$ & Lower bound on $\pi_{\mathrm{mode}}$, loose when errors scatter. \\
\bottomrule
\end{tabular}}
\end{table}

The formulas price each behavior; reading the methods as estimands then assigns each its ceiling. Coverage escapes every ceiling because it asks only whether a correct sample exists; self-consistency and best-of-$n$ meet the modal-answer wall because they read the answer distribution; and only benchmark-mean estimation meets the correlation ceiling, the one place the design effect of Section~\ref{sec:neff} truly binds.

\begin{table}[H]
\centering
\caption{Test-time-scaling methods as estimands, with the ceiling that binds each.}
\label{tab:methods}
{\small
\setlength{\tabcolsep}{5pt}
\renewcommand{\arraystretch}{1.25}
\begin{tabular}{@{}llll@{}}
\toprule
\textbf{Method} & \textbf{Reads} & \textbf{Estimand} & \textbf{Ceiling} \\
\midrule
Coverage / $\mathrm{pass}@n$ (with a verifier) & any correct sample & existence indicator & reachability, no $1/\rho$ \\
Self-consistency / plurality & most common answer & mode of $p_q$ & $\pi_{\mathrm{mode}}$ \\
Best-of-$n$ (reward model) & top-scored sample & argmax of reward & reward-model ranking \\
\rowcolor{cnavy!9} Benchmark-mean estimate & success fraction $\hat p$ & sample mean & $\neff\le1/\rho_b$ \\
Single sample / greedy & one answer & single draw & n/a \\
\bottomrule
\end{tabular}}
\end{table}

\section{Conclusion}
\label{sec:conclusion}

More sampling makes coverage climb while selection stalls; what separates them is the \emph{identifiability gap}, the solvable problems whose answer a vote never returns. Selection stalls because it meets the modal ceiling $\pi_{\mathrm{mode}}$, fixed by how often the most common answer happens to be right: on the dependent-draw log of \citet{beeching2024scaling} the $256$ tries per problem reduce to about thirteen, coverage reaches $0.88$ while selection is right for only $0.45$, and drawing more only sharpens a confident error. A different use meets a different limit: estimating a benchmark mean is capped by the correlation ceiling $1/\rho_b$, so with $\hat\rho_b\approx0.4$--$0.6$ on the released logs of \citet{brown2024monkeys}, ten thousand attempts on a single problem buy the precision of about two, while the seed-to-seed term stays near zero ($\hat\rho_w\approx0.0007$), leaving coverage with no within-problem ceiling. One sample count thus stands in for three different things at once: difficulty spread between problems, dependence between runs, and the collapse of answers onto a mode.

\textbf{When to stop.} Because each draw costs compute and both ceilings are low, the budget that pays is small and set by the goal: about $1/\rho_b$ samples to estimate a benchmark mean, on the order of the effective number of answers to select one, and no limit for coverage where a verifier can pick the correct sample out. Reporting the effective number of samples beside the nominal $n$, a closed form solved for rather than searched \citep{bay2026kore}, says how much of a budget actually counts, and where it stops paying. The bottleneck in test-time scaling has moved from generating a correct answer to recognizing one: coverage shows the answers are already present, the modal ceiling shows that more sampling will not surface them, and the compute that extra draws cannot use is better spent correlating answers less and choosing among them better, not drawing more.

\bibliographystyle{unsrtnat}
\bibliography{references}

\clearpage
\appendix
\titleformat{\section}{\normalfont\large\bfseries}{Appendix \thesection:}{0.45em}{}

\section{Elementary derivations}
\label{app:derivations}

The body states the propositions and leans on these short calculations without pausing for them; each one underwrites a result above.

\paragraph{Beta-binomial moments and intraclass correlation.}
Take the latent rate $\theta\sim\mathrm{Beta}(\alpha,\beta)$ with attempts $Y_i\mid\theta\sim\text{i.i.d.\ Bernoulli}(\theta)$. The Beta has mean and variance
\[
s = \E[\theta] = \frac{\alpha}{\alpha+\beta}, \qquad
\Var(\theta) = \frac{\alpha\beta}{(\alpha+\beta)^2(\alpha+\beta+1)} = \frac{s(1-s)}{\alpha+\beta+1}.
\]
Two distinct attempts share only the rate $\theta$, so their covariance is the variance of that shared rate: for $i\ne j$, $\Cov(Y_i,Y_j) = \E[\theta^2]-s^2 = \Var(\theta)$. Dividing by $s(1-s)$ gives the intraclass correlation of Equation~\eqref{eq:icc},
\[
\rho = \frac{\Var(\theta)}{s(1-s)} = \frac{1}{\alpha+\beta+1},
\]
which does not depend on $n$. Inverting, a target mean $s$ and correlation $\rho$ are realized by $\alpha = s(1-\rho)/\rho$ and $\beta = (1-s)(1-\rho)/\rho$, the map the figures use to set a Beta from $(s,\rho)$.

\paragraph{Zero-success probability and its tail.}
The chance that none of $n$ attempts succeeds is the Beta average of $(1-\theta)^n$, a standard beta integral:
\[
\Prob[K=0] = \E_\theta\big[(1-\theta)^n\big]
= \int_0^1 (1-\theta)^n \frac{\theta^{\alpha-1}(1-\theta)^{\beta-1}}{B(\alpha,\beta)}\,d\theta
= \frac{B(\alpha,\beta+n)}{B(\alpha,\beta)}.
\]
Written with gamma functions this is $\frac{\Gamma(\alpha+\beta)}{\Gamma(\beta)}\cdot\frac{\Gamma(\beta+n)}{\Gamma(\alpha+\beta+n)}$. For large $n$ the gamma ratio decays polynomially, $\Gamma(\beta+n)/\Gamma(\alpha+\beta+n)\sim n^{-\alpha}$ (Stirling, $\Gamma(x+a)/\Gamma(x+b)\sim x^{a-b}$), so
\[
\Prob[K=0] \sim \frac{\Gamma(\alpha+\beta)}{\Gamma(\beta)}\, n^{-\alpha} \qquad (n\to\infty).
\]
This is the power-law tail of Proposition~\ref{prop:powerlaw}: coverage approaches its limit polynomially, not exponentially.

\paragraph{The selection plateau.}
Plurality returns the most frequent answer, and the answer counts over $n$ attempts are $\mathrm{Multinomial}(n,p_q)$. By the law of large numbers the empirical frequencies converge to $p_q$ almost surely, so the empirical mode converges to the true mode $a^\star_q=\arg\max_a p_q(a)$, and per-problem plurality accuracy to $\mathbf 1\{a^\star_q=c_q\}$. Averaging over problems gives the modal-hit rate $\pi_{\mathrm{mode}}$ of Equation~\eqref{eq:vote}.

The two-answer case recovers the majority vote. With only a correct answer and one error, write $\theta$ for the per-attempt correctness; then plurality is majority and $\hat p=K/n\to\theta$ almost surely, so in the limit the majority is right precisely if $\theta>\tfrac12$. Its variance tends to $\Var(\hat p)\to\rho\,s(1-s)$, and approximating $\theta$ as $\mathcal N\big(s,\,\rho\,s(1-s)\big)$ gives the closed form
\[
\Prob[\theta>\tfrac12] \approx \Phi\!\left(\frac{s-\tfrac12}{\sqrt{\rho\,s(1-s)}}\right)
\]
of Corollary~\ref{cor:condorcet}. Because $\theta>\tfrac12$ pushes every rival answer under $\tfrac12$, hence under the correct one, this majority bound never exceeds $\pi_{\mathrm{mode}}$.

\paragraph{Two-stage design effect.}
Pool attempts over problems, and let $\mu_i$ be the success rate of problem $i$, with between-problem correlation $\rho_b = \Var(\mu_i)/[s(1-s)]$ and a common within-problem correlation $\rho_w$ given the problem. By the law of total covariance, two same-problem attempts $j\ne j'$ have
\[
\Cov(Y_{ij},Y_{ij'}) = \E\!\big[\rho_w\,\mu_i(1-\mu_i)\big] + \Var(\mu_i).
\]
Since $\E[\mu_i(1-\mu_i)] = s(1-s) - \Var(\mu_i) = s(1-s)(1-\rho_b)$, the right-hand side is $s(1-s)\big[\rho_b + (1-\rho_b)\rho_w\big]$. Dividing by $s(1-s)$ gives the pooled correlation
\[
\rho = \rho_b + (1-\rho_b)\rho_w
\]
of Equation~\eqref{eq:decomp}, and the same-problem design effect is $1+(n-1)\rho$ \citep{cochran1977sampling}.

\section{Reproducibility}
\label{app:repro}

Every number, table, and figure in this work regenerates from a clean checkout. The propositions are checked numerically by \texttt{scripts/verify\_math.py}, which verifies the identities of Sections~\ref{sec:neff}--\ref{sec:practice} (including the variance identity~\eqref{eq:varK}, the coverage monotonicity of Proposition~\ref{prop:cov_mono}, the modal-answer selection ceiling~\eqref{eq:vote} with its anti-scaling corollary and the majority-vote lower bound, and the two-stage decomposition~\eqref{eq:decomp}) against Monte Carlo simulation; all checks pass. The five model-based figures are generated by \texttt{scripts/make\_figures.py} from closed forms and fixed-seed simulation; the correlated attempts use the de Finetti representation $\theta\sim\mathrm{Beta}(\alpha,\beta)$, $Y_i\mid\theta\sim\text{Bernoulli}(\theta)$, with $(\alpha,\beta)$ set from $(s,\rho)$ by Appendix~\ref{app:derivations}. The between-problem figure and Table~\ref{tab:empirical} are produced by \texttt{scripts/analyze\_brown.py}, which downloads the public sampling logs of \citet{brown2024monkeys}, estimates $\hat\rho_b$ with a clustered (problem-level) bootstrap $95\%$ interval ($10^4$ resamples, fixed seed), computes the coverage and self-consistency curves, and writes a small summary that the figure script reads. The within-problem figure and the $\hat\rho_w$, $\hat\rho_b$, and decomposition numbers of Section~\ref{sec:gap} are produced by \texttt{scripts/analyze\_rhow.py}, which downloads the five-session best-of-$n$ log of \citet{beeching2024scaling}, grades every completion with \texttt{math-verify} (the verifier behind the dataset's own labels, whose reported single-sample accuracy it reproduces to about a point), estimates the between- and within-problem correlations with problem-clustered bootstrap intervals, and computes the within-session coverage and plurality curves; the graded per-problem counts are cached so the figure rebuilds without re-downloading or re-grading the multi-hundred-megabyte log. The Python environment is pinned in \texttt{pyproject.toml} (\texttt{uv sync}); \texttt{make all} regenerates everything. Seeds are fixed, so the results are reproducible. The code, the cached result summaries, and the manuscript source are available at \url{https://github.com/bay-yearick-lab/sampling-ceilings}.

\end{document}